\documentclass{article}

\usepackage{microtype}
\usepackage{graphicx}
\usepackage{subcaption}
\usepackage{booktabs} 
\usepackage{enumitem}
\usepackage{hyperref}

\usepackage[preprint]{icml2026}

\usepackage{amsmath}
\usepackage{amssymb}
\usepackage{mathtools}
\usepackage{amsthm}
\usepackage{booktabs}
\usepackage{tcolorbox}
\tcbuselibrary{listings, breakable}

\usepackage[capitalize,noabbrev]{cleveref}

\theoremstyle{plain}
\newtheorem{theorem}{Theorem}[section]
\newtheorem{proposition}[theorem]{Proposition}

\theoremstyle{definition}
\newtheorem{definition}[theorem]{Definition}

\theoremstyle{remark}

\usepackage[textsize=tiny]{todonotes}

\icmltitlerunning{Clipping-Free Policy Optimization for Large Language Models}

\begin{document}

\twocolumn[
  \icmltitle{Clipping-Free Policy Optimization for Large Language Models}



  \icmlsetsymbol{equal}{*}

\begin{icmlauthorlist}
  \icmlauthor{Ömer Veysel Çağatan}{kuisai}
  \icmlauthor{Barış Akgün}{kuisai,koc}
  \icmlauthor{Gözde Gül Şahin}{kuisai,koc}
  \icmlauthor{Xuandong Zhao}{berkeley}
\end{icmlauthorlist}

\icmlaffiliation{kuisai}{KUIS AI Center, Koç University, Istanbul, Türkiye}
\icmlaffiliation{koc}{Koç University, Istanbul, Türkiye}
\icmlaffiliation{berkeley}{University of California, Berkeley, CA, USA}

\icmlcorrespondingauthor{Ömer Veysel Çağatan}{ocagatan19@ku.edu.tr}
\icmlcorrespondingauthor{Xuandong Zhao}{xuandongzhao@berkeley.edu}

  \icmlkeywords{Machine Learning, ICML}

  \vskip 0.3in
]



\printAffiliationsAndNotice{}  

\begin{abstract}
Reinforcement learning has become central to post-training large language models, yet dominant algorithms rely on clipping mechanisms that introduce optimization issues at scale, including zero-gradient regions, reward hacking, and training instability. We propose Clipping-Free Policy Optimization (CFPO), which replaces heuristic clipping with a convex quadratic penalty derived from Total Variation divergence constraints, yielding an everywhere-differentiable objective that enforces stable policy updates without hard boundaries. We evaluate CFPO across both reasoning and alignment settings. In reasoning, CFPO matches clipping-based methods on downstream benchmarks while extending the stable training regime. In alignment, CFPO mitigates verbosity exploitation and reduces capability degradation, while achieving competitive instruction-following performance. CFPO requires only a one-line code change and no additional hyperparameters. Our results suggest that CFPO is a promising drop-in alternative to clipping-based methods for LLM post-training.

\end{abstract}

\section{Introduction}

Reinforcement learning (RL) has become a central component of large language model (LLM) post-training. Early work demonstrated that RL from human feedback (RLHF) could align models with human preferences and instructions~\citep{christiano2023deepreinforcementlearninghuman,stiennon2022learningsummarizehumanfeedback,bai2022constitutionalaiharmlessnessai,ouyang2022traininglanguagemodelsfollow}, and more recent efforts have shown that RL with verifiable rewards (RLVR) can elicit complex reasoning behaviors~\citep{deepseekai2025deepseekr1incentivizingreasoningcapability,openai2024openaio1card,deepmind2025gemini25pro}. These successes have established RL as an essential stage in modern LLM training pipelines.

The dominant algorithms for LLM post-training, including PPO~\citep{schulman2017proximalpolicyoptimizationalgorithms}, GRPO~\citep{shao2024deepseekmath}, and their variants, rely on clipped surrogate objectives to stabilize policy updates. Clipping serves as a computationally efficient approximation to trust region constraints: probability ratios between the current and previous policy are clipped to a narrow range, removing incentives for updates that would change the policy too drastically. This mechanism has enabled scaling RL to large language models, but it remains a heuristic, and its limitations become increasingly apparent at scale.

The core issue is a discontinuity induced by hard clipping in the optimization landscape. Within the clipping range, the objective reduces to unconstrained advantage maximization. Beyond the clipping boundary, gradients vanish entirely. This combination creates pathological dynamics: models learn to exploit superficial correlates of reward such as verbosity~\citep{gao2022scalinglawsrewardmodel}, rapid policy drift degrades capabilities acquired during pretraining~\citep{ouyang2022traininglanguagemodelsfollow,casper2023openproblemsfundamentallimitations}, and zero-gradient regions contribute to entropy collapse and training instability~\citep{liu2025understandingr1zeroliketrainingcritical,huang2025lowprobabilitytokenssustainexploration,yu2025dapoopensourcellmreinforcement,qwen3technicalreport,minimax2025minimaxm1scalingtesttimecompute}. These failures manifest across model scales and training configurations, suggesting they are intrinsic to clipping rather than incidental to specific implementations.

Recent work has proposed targeted modifications: asymmetric clipping bounds~\citep{yu2025dapoopensourcellmreinforcement}, modified advantage normalization~\citep{liu2025understandingr1zeroliketrainingcritical}, dynamic clipping thresholds~\citep{yang2025dcpodynamicclippingpolicy}, and auxiliary regularization~\citep{wang2025lambdagrpounifyinggrpoframeworks}. While these methods mitigate specific failure modes, they treat clipping as a mechanism to be patched rather than replaced, introducing additional hyperparameters while leaving the fundamental discontinuity intact.

In this paper, we propose Clipping-Free Policy Optimization (CFPO), a principled alternative that eliminates clipping entirely. CFPO builds on Simple Policy Optimization (SPO)~\citep{xie2025simplepolicyoptimization}, which replaces clipping with a convex quadratic penalty derived from Total Variation divergence constraints~\citep{queeney2021generalizedproximalpolicyoptimization}. Unlike clipping, this objective is everywhere differentiable: the quadratic term applies a restoring force proportional to deviation from the old policy, rather than zeroing gradients beyond a threshold. As we show empirically, this produces more stable optimization dynamics.

We evaluate CFPO across diverse post-training settings, substituting it for clipped objectives while holding all other training details constant. Our experiments span Qwen2.5~\citep{qwen2.5} models from 1.5B to 7B parameters for reasoning tasks and Llama3-8B~\citep{Llama3modelcard} for alignment, using GRPO and RLOO~\citep{ahmadian2024basicsrevisitingreinforcestyle} as baselines respectively. Across settings, CFPO exhibits more conservative optimization dynamics: slower initial reward growth but sustained progress, lower clipping ratios, and gradual entropy consumption rather than rapid collapse. These dynamics yield concrete improvements:

\begin{itemize}[leftmargin=*, itemsep=0pt, topsep=0pt] 
   \item \textbf{Stable reasoning training.} GRPO becomes unstable around 8 training iterations and completely collapses by 16 across most configurations. CFPO extends the stable training regime, with controlled entropy decay, policy KL, and clipping ratio throughout, while matching GRPO in reasoning accuracy on MATH500~\citep{hendrycks2021measuring}, AIME24~\citep{aime2024}, GSM8K~\citep{cobbe2021gsm8k}, and GPQA-Diamond~\citep{rein2023gpqagraduatelevelgoogleproofqa} datasets.
   \item \textbf{Robust alignment.} CFPO mitigates verbosity exploitation, improving length-controlled AlpacaEval~\citep{dubois2024length} by 4 points while achieving competitive performance on Arena-Hard~\citep{li2024crowdsourceddatahighqualitybenchmarks} and MT-Bench~\citep{zheng2023judging}. CFPO also reduces alignment tax from 12--16\% to 4--5\% on OpenLLM leaderboard~\citep{open-llm-leaderboard} tasks.
\end{itemize}

Overall, our results suggest that CFPO is a promising drop-in alternative to clipping-based methods, offering more stable training without sacrificing downstream performance while requiring only a one-line code change and no additional hyperparameters.

\section{Background}

\subsection{Proximal Policy Optimization}

Proximal Policy Optimization (PPO) \citep{schulman2017proximalpolicyoptimizationalgorithms} is a policy gradient method that stabilizes training by constraining how much the policy can change in a single update. Given a policy $\pi_\theta$, PPO maximizes a clipped surrogate objective:
\begin{equation}
\begin{split}
    \mathcal{J}_{\text{PPO}}(\theta) = \mathbb{E}_{(s_t, a_t) \sim \pi_{\theta_{\text{old}}}} \Big[ \min \big( r_t(\theta) \hat{A}_t, \\
    \text{clip}(r_t(\theta), 1-\epsilon, 1+\epsilon) \hat{A}_t \big) \Big],
\end{split}
\end{equation}
where $r_t(\theta) = \pi_\theta(a_t|s_t) / \pi_{\theta_{\text{old}}}(a_t|s_t)$ is the probability ratio between the current and old policy, $\hat{A}_t$ is the estimated advantage, and $\epsilon$ is a hyperparameter controlling the clipping range. The clipping mechanism removes incentives for pushing the ratio outside $[1-\epsilon, 1+\epsilon]$, approximating a trust region constraint without the computational overhead of second-order methods like TRPO \citep{schulman2017trustregionpolicyoptimization}. In practice, PPO also employs a learned value function to estimate advantages via Generalized Advantage Estimation \citep{schulman2018highdimensionalcontinuouscontrolusing}.

\begin{figure*}[t]
    \centering
    \includegraphics[width=0.95\textwidth]{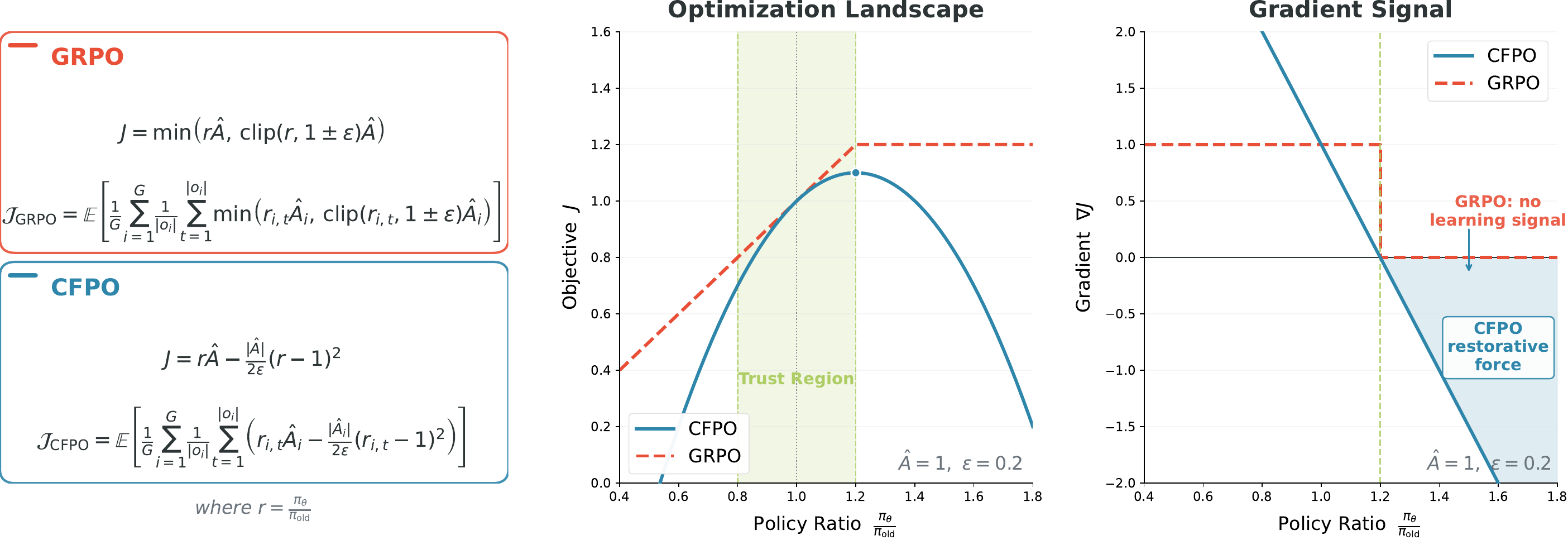}
    \caption{Optimization objective and gradient of GRPO and CFPO as functions of the policy ratio $r = \pi / \pi_{\text{old}}$, shown for advantage $A = 1$ and trust-region width $\epsilon = 0.2$. 
GRPO becomes flat once $r$ exits the trust region, resulting in zero gradient beyond the clipping boundary. 
CFPO instead applies a convex quadratic penalty in $r$, yielding a continuous restoring gradient that pulls $r$ back toward the trust region. 
This difference highlights why CFPO maintains stable learning signals while GRPO can stall when updates push $r$ outside the trust region.}
    \label{fig:method figure}
\end{figure*}
\subsection{Group Relative Policy Optimization}

Group Relative Policy Optimization (GRPO) adapts PPO to the language model setting \citep{shao2024deepseekmath}. In standard PPO, advantages are computed using a learned value function, requiring a critic network of comparable size to the policy---a substantial memory burden for large language models. GRPO eliminates this requirement by estimating advantages through comparisons within groups of sampled responses to the same prompt.

For each prompt $q$, GRPO samples a group of responses $\{o_1, o_2, \ldots, o_G\}$ from the current policy and computes advantages based on relative rewards within the group:
\begin{equation}
    \hat{A}_{i} = \frac{r_i - \text{mean}(\{r_j\}_{j=1}^G)}{\text{std}(\{r_j\}_{j=1}^G)},
\end{equation}
where $r_i$ is the reward for $o_i$. The GRPO objective is:
\begin{align}
    \mathcal{J}_{\text{GRPO}}(\theta) = \mathbb{E}_{q, \{o_i\}_{i=1}^G} \Bigg[ \frac{1}{G} & \sum_{i=1}^G \frac{1}{|o_i|} \sum_{t=1}^{|o_i|} \Big( L_{\text{clip}}(r_{i,t}, \hat{A}_i) \notag
    \\
    & - \beta \, \mathbb{D}_{\text{KL}}[\pi_\theta \| \pi_{\text{ref}}] \Big) \Bigg],
\end{align}
where $L_{\text{clip}}$ denotes PPO's clipped objective, $r_{i,t}(\theta) = \pi_\theta(o_{i,t}|q, o_{i,<t}) / \pi_{\theta_{\text{old}}}(o_{i,t}|q, o_{i,<t})$ is the per-token probability ratio, and $\beta$ controls KL regularization against a reference policy $\pi_{\text{ref}}$.
\subsection{Simple Policy Optimization}

Simple Policy Optimization (SPO)~\citep{xie2025simplepolicyoptimization} provides a principled alternative to PPO’s clipping mechanism. Although PPO’s clipping is often viewed as heuristic, it can be interpreted as approximately enforcing a trust region defined by Total Variation (TV) divergence between successive policies~\citep{queeney2021generalizedproximalpolicyoptimization}. Under this view, PPO implicitly optimizes the policy objective subject to a constraint on the expected deviation of probability ratios from unity.

A key motivation for SPO is that TV divergence constraints induce a strictly larger feasible policy space than the KL divergence constraints used in TRPO. Moreover, optimizing within the TV-constrained space yields a tighter policy improvement lower bound than the corresponding KL-constrained formulation. We state these results below and defer formal proofs to Appendix~\ref{app:theory}.

\begin{proposition}[TV Solution Space Contains KL Solution Space]
\label{prop:tv_contains_kl_main}
Let $\Omega_{\mathrm{TV}}$ and $\Omega_{\mathrm{KL}}$ denote the policy sets satisfying per-state TV and KL divergence constraints, respectively. If $\delta_{\mathrm{TV}} \ge \sqrt{\delta_{\mathrm{KL}}/2}$, then
$\Omega_{\mathrm{KL}} \subset \Omega_{\mathrm{TV}}$.
\end{proposition}

\begin{theorem}[TV-Constrained Policy Improvement]
\label{thm:tv_superior_main}
Let $\mathcal{L}_\pi^{\mathrm{TV}}$ and $\mathcal{L}_\pi^{\mathrm{KL}}$ denote the standard policy improvement lower bounds with TV and KL penalties~\citep{xie2025simplepolicyoptimization}. Let $\tilde{\pi}^*_{\mathrm{TV}}$ and $\tilde{\pi}^*_{\mathrm{KL}}$ be their respective maximizers over $\Omega_{\mathrm{TV}}$ and $\Omega_{\mathrm{KL}}$. For $
\delta_{\mathrm{TV}} \ge \sqrt{\delta_{\mathrm{KL}}/2},
$
the following holds:
\begin{equation}
\mathcal{L}_\pi^{\mathrm{TV}}(\tilde{\pi}^*_{\mathrm{TV}})
\;\ge\;
\mathcal{L}_\pi^{\mathrm{KL}}(\tilde{\pi}^*_{\mathrm{KL}}).
\end{equation}
\end{theorem}

Unlike PPO, which enforces the trust region via clipping and yields zero gradients for samples outside the clipping range, SPO incorporates the constraint directly using a quadratic penalty. The resulting objective is
\begin{equation}
\mathcal{J}_{\mathrm{SPO}}(\theta)
=
\mathbb{E}_{(s_t,a_t)\sim\pi_{\theta_{\mathrm{old}}}}
\!\left[
r_t \hat{A}_t
-
\frac{|\hat{A}_t|}{2\epsilon}(r_t-1)^2
\right],
\end{equation}
where $r_t = \pi_\theta(a_t \mid s_t) / \pi_{\theta_{\mathrm{old}}}(a_t \mid s_t)$. This objective is convex and differentiable in $r_t$, and its maximizer satisfies $r_t^* = 1 + \mathrm{sign}(\hat{A}_t)\epsilon$,
corresponding to an update at the trust region boundary while retaining nonzero gradients for all samples.
\section{Methodology}

\subsection{Clipping-Free Policy Optimization}

Motivated by the success of SPO~\citep{xie2025simplepolicyoptimization} in simulation environments, we propose Clipping-Free Policy Optimization (CFPO) as an adaptation to the language model setting. CFPO retains the critic-free design common to modern LLM post-training while replacing the clipped surrogate objective with SPO's quadratic penalty. This substitution requires only a one-line code change, making CFPO a drop-in replacement in existing pipelines.

The core modification is straightforward. Where clipping-based methods enforce the trust region via:
\begin{equation}
L_{\text{clip}} = \min\left(r_{i,t} \hat{A}_i, \;\text{clip}(r_{i,t}, 1-\epsilon, 1+\epsilon) \hat{A}_i\right),
\end{equation}
CFPO instead applies a quadratic penalty:
\begin{equation}
L_{\text{CFPO}} = r_{i,t} \hat{A}_i - \frac{|\hat{A}_i|}{2\epsilon}(r_{i,t} - 1)^2,
\end{equation}
where $r_{i,t}(\theta) = \pi_\theta(o_{i,t}|q, o_{i,<t}) / \pi_{\theta_{\text{old}}}(o_{i,t}|q, o_{i,<t})$ is the per-token probability ratio and $\epsilon$ controls the trust region width. As shown in Figure~\ref{fig:method figure}, this yields a convex, everywhere-differentiable objective: rather than zeroing gradients when the ratio exits the clipping range, CFPO provides a continuous restoring force that pulls the policy back toward the trust region. The penalty is minimized at $r_{i,t} = 1$ and grows quadratically with deviation, while scaling by $|\hat{A}_i|$ ensures that larger advantages permit proportionally larger updates.

The full CFPO objective is:
\begin{equation}
\begin{split}
\mathcal{J}_{\text{CFPO}}(\theta) = \mathbb{E}_{q, \{o_i\}_{i=1}^G} \Bigg[ \frac{1}{G} \sum_{i=1}^G \frac{1}{|o_i|} \sum_{t=1}^{|o_i|} \Big( r_{i,t} \hat{A}_i \\
- \frac{|\hat{A}_i|}{2\epsilon} (r_{i,t} - 1)^2 - \beta \, \mathbb{D}_{\text{KL}}[\pi_\theta \| \pi_{\text{ref}}] \Big) \Bigg],
\end{split}
\end{equation}
where $\beta$ is the KL regularization coefficient against a reference policy $\pi_{\text{ref}}$. Since CFPO modifies only the surrogate objective, it is agnostic to how advantages $\hat{A}_i$ are estimated. We exploit this modularity to evaluate CFPO with two distinct advantage estimators standard to different post-training settings.

\subsection{Advantage Estimation}

\paragraph{Group-Relative Advantages.} For reasoning tasks, we follow GRPO~\citep{shao2024deepseekmath}. Given a prompt $q$, we sample a group of $G$ responses $\{o_1, o_2, \ldots, o_G\}$ from the current policy and compute group-normalized advantages:
\begin{equation}
\hat{A}_{i} = \frac{R_i - \text{mean}(\{R_j\}_{j=1}^G)}{\text{std}(\{R_j\}_{j=1}^G)},
\end{equation}
where $R_i$ is the reward for response $o_i$. Group-relative advantages suit reasoning tasks where rewards are verifiable (e.g., correctness of mathematical solutions), and normalization by standard deviation helps stabilize learning when reward magnitudes vary across different problems.

\paragraph{Leave-One-Out Advantages.} For alignment tasks, we adopt the REINFORCE Leave-One-Out (RLOO) estimator~\citep{ahmadian2024basicsrevisitingreinforcestyle}. Given $K$ sampled responses per prompt, RLOO computes the advantage as:
\begin{equation}
\hat{A}_{i} = R_i - \frac{1}{K-1} \sum_{j \neq i} R_j,
\end{equation}
where each sample uses the remaining $K-1$ samples as an unbiased baseline estimate, functioning as a parameter-free value function. Unlike group-relative advantages, RLOO does not normalize by standard deviation, resulting in different gradient scaling behavior.

\vspace{0.5em}
The use of two distinct advantage estimators allows us to assess whether CFPO's improvements stem from the quadratic penalty itself or from interactions with specific estimation choices. As we show in our experiments, CFPO exhibits consistent stability benefits across both settings, suggesting that replacing clipping with a convex penalty is broadly beneficial regardless of how advantages are computed.

\begin{figure*}[t]
    \centering
    \includegraphics[width=\textwidth]{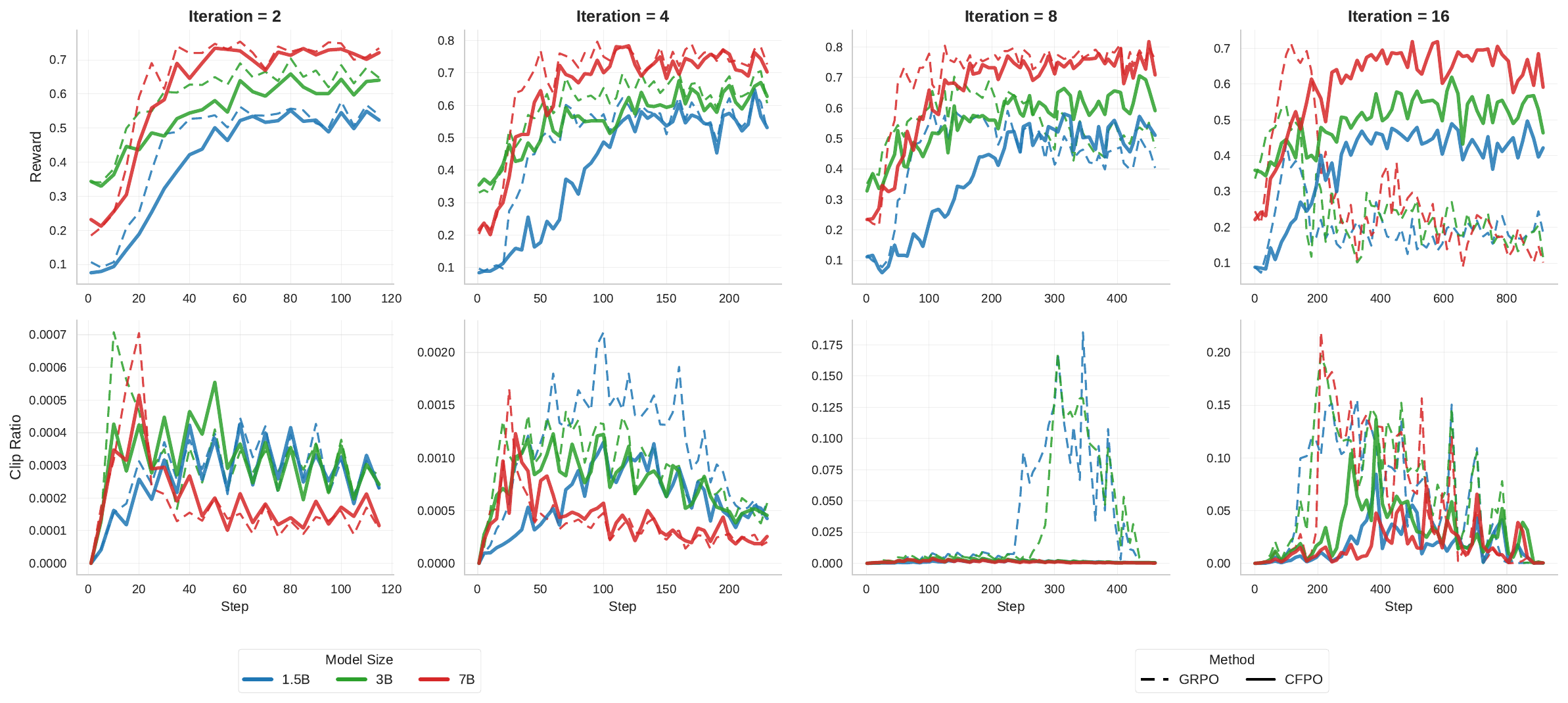}
    \caption{Training dynamics of CFPO vs. GRPO under increasing off-policy pressure.
Reward (top) and clip ratio (bottom) trajectories for Qwen2.5 models trained with different numbers of iterations per update (columns). GRPO (dashed) exhibits faster early reward gains but increasingly large and unstable updates as iterations grow, reflected in rising clip ratios and eventual training collapse at higher iteration counts ($\geq$ 8 for most models). In contrast, CFPO (solid) progresses more conservatively, maintaining consistently low clip ratios and stable training across extended horizons, while ultimately reaching comparable reward levels. These dynamics illustrate the trade-off between optimization aggressiveness and stability in off-policy post-training, and highlight CFPO’s robustness to repeated sample reuse.}
    \label{fig:model_metrics_comparison-trl}
\end{figure*}

\section{Experimental Setup}
\subsection{Reasoning}

\paragraph{Training Setup.}
We follow the training setup of~\citet{zhao2025learningreasonexternalrewards} and train GRPO and CFPO using the Open-R1/TRL framework~\citep{openr1,vonwerra2022trl} on the training split of the MATH dataset~\citep{hendrycks2021measuring}, which contains 7{,}500 problems. We use Qwen2.5-1.5B, Qwen2.5-3B, and Qwen2.5-7B~\citep{qwen2.5} as backbone models, experimenting with both base and Instruct variants. All models are trained using a chat-style prompting format.
Since the base models exhibit limited instruction-following ability prior to training, we do not require explicit separation between intermediate reasoning and final answers. Each training update processes 128 problems, and we generate 3 candidate solutions per problem. Unless otherwise specified, we use a KL penalty coefficient of $\beta = 0.0$. For some ablations, we also use the \texttt{verl} framework~\citep{sheng2024hybridflow} with Qwen2.5-3B.

\paragraph{Off-Policy Mechanisms in \texttt{verl} and TRL.}
Comparing GRPO and CFPO requires off-policy settings where the effectiveness of trust region methods can be meaningfully evaluated; in on-policy RL, both methods reduce to simple advantage maximization. In the traditional deep RL literature~\citep{schulman2017proximalpolicyoptimizationalgorithms, engstrom2020implementationmattersdeeppolicy,huang2021cleanrlhighqualitysinglefileimplementations,xie2025simplepolicyoptimization}, off-policy behavior is typically introduced through two mechanisms: sample reuse and mini-batch gradient updates. 

However, modern RL frameworks for LLMs differ in their implementations of policy gradients. \texttt{verl} follows the standard deep RL structure for off-policy training, while TRL, which Open-R1 is based on, deviates from this approach. TRL follows the GRPO recipe and only supports sample reuse (also referred to as iterations) without mini-batch updates. This discrepancy has practical consequences: for instance, clipping may be observed in \texttt{verl} but not in TRL for certain recipes that have different mini-batch size than batch size. Note that when iteration is set to 1 and mini-batch size equals batch size, both frameworks operate in the on-policy regime, where GRPO and CFPO are equivalent.

To fairly compare objectives across frameworks, we conduct experiments with the following configurations. For TRL-based models, we vary iterations as powers of 2 up to 16, across 3 model sizes, 2 policy loss types (GRPO and CFPO), and 2 model bases (base and instruct), yielding 48 models. For \texttt{verl}, we vary iterations up to 8, with 4 different batch ratios (batch size divided by mini-batch size, corresponding to the number of local updates), and 2 policy loss types, yielding 32 models. In total, we train 80 models across both frameworks.

\paragraph{Evaluation.}
We adopt the same chat-style prompting format used during training for evaluation. We use sampling-based decoding with temperature $0.6$ and top-$p$ $0.95$ for all evaluations. We evaluate on MATH500~\citep{hendrycks2021measuring} and GSM8K~\citep{cobbe2021gsm8k}, and AIME24~\citep{aime2024} for math reasoning, and GPQA-Diamond~\citep{rein2023gpqagraduatelevelgoogleproofqa} for scientific reasoning, all using the \texttt{lighteval} library~\citep{lighteval}.

\subsection{RLHF}

\paragraph{Training Setup.}
To compare CFPO against RLOO (critic-free PPO with reinforce leave-one-out advantage estimation) in standard RLHF pipeline, we train models using the RLOO advantage estimation~\citep{ahmadian2024basicsrevisitingreinforcestyle}. We use the supervised fine-tuning (SFT) and reward models provided by the OpenRLHF repository~\citep{hu2024openrlhf}, both based on Llama-3 backbones~\citep{Llama3modelcard}. RLHF training additionally uses the prompt collections provided by OpenRLHF. All RLHF experiments use default OpenRLHF hyperparameters, with $k=2$ rollouts per prompt. For KL-free RLHF settings, we explicitly set the KL coefficient to zero; otherwise, all parameters follow the default configuration.

\paragraph{Evaluation.}
We evaluate RLHF-trained models on a range of instruction-following benchmarks, including AlpacaEval 2.0~\citep{dubois2024length}, Arena-Hard v0.1~\citep{li2024crowdsourceddatahighqualitybenchmarks}, MT-Bench~\citep{zheng2023judging}, and IFEval~\citep{zhou2023instructionfollowingevaluationlargelanguage}. AlpacaEval 2.0 and Arena-Hard v0.1 are judged using GPT-4.1~\citep{openai2025gpt4.1}, while MT-Bench evaluations use GPT-4~\citep{openai2024gpt4technicalreport}. To assess whether RLHF preserves general model capabilities, we additionally evaluate on a subset of tasks from the OpenLLM Leaderboard~\citep{open-llm-leaderboard}, specifically ARC Challenge~\citep{clark2018think}, GSM8K~\citep{cobbe2021gsm8k}, HellaSwag~\citep{zellers-etal-2019-hellaswag}, MMLU~\citep{hendrycks2020measuring}, TruthfulQA~\citep{lin2022truthfulqa}, and Winogrande~\citep{levesque2012winograd}.

\begin{table*}[t]
\centering
\caption{Instruction-following and preference alignment results of CFPO and RLOO on Llama3-8B. 
We report performance on preference-based benchmarks (Arena-Hard, AlpacaEval, MT-Bench) and instruction-following (IFEval), including unregularized variants with KL coefficient set to zero. 
Across benchmarks, CFPO maintains strong alignment behavior while exhibiting reduced sensitivity to verbosity compared to RLOO, as reflected by differences between raw and length-controlled evaluations.}
\label{tab:alignment}
\begin{tabular}{lcccccc}
\toprule
\textbf{Method} 
& \textbf{Arena-Hard} 
& \textbf{AlpacaEval-LC} 
& \textbf{AlpacaEval-WR}
& \textbf{MT-Bench} 
& \textbf{IFEval} \\
\midrule
Llama3-8B-SFT & 6.5 & 5.22 & 3.52 & 7.84 & 59.59 \\
RLOO & 21.8 & 7.25 & 18.17 & 7.90 & 47.00 \\
RLOO (KL=0) & 22.1 & 8.15 & 16.39 & 7.86 & 44.12 \\
CFPO & 20.1 & 11.26 & 17.08 & 7.87 & 55.64 \\
CFPO (KL=0) & 22.5 & 11.55 & 16.08 & 8.04 & 54.20 \\
\bottomrule
\end{tabular}
\end{table*}

\begin{table*}[h]
\centering
\caption{Downstream capability evaluation of CFPO and RLOO after RLHF on Llama3-8B, measured on the OpenLLM Leaderboard tasks. Comparisons against the SFT baseline illustrate how different alignment objectives affect general-purpose capabilities. CFPO preserves substantially more of the base model’s capabilities across tasks, whereas RLOO induces broader degradation.}
\label{tab:capability}
\begin{tabular}{lccccccc}
\toprule
\textbf{Method} & \textbf{ARC} & \textbf{GSM8K} & \textbf{HellaSwag} & \textbf{MMLU} & \textbf{TruthfulQA} & \textbf{Winogrande} & \textbf{Avg} \\
\midrule
Llama3-8B-SFT & 0.525 & 0.762 & 0.584 & 0.627 & 0.401 & 0.719 & 0.603 \\
RLOO & 0.382 & 0.813 & 0.435 & 0.626 & 0.392 & 0.536 & 0.531 \\
RLOO (KL=0) & 0.374 & 0.683 & 0.408 & 0.624 & 0.406 & 0.535 & 0.505 \\
CFPO & 0.432 & 0.797 & 0.542 & 0.631 & 0.441 & 0.617 & 0.577 \\
CFPO (KL=0) & 0.431 & 0.789 & 0.539 & 0.631 & 0.428 & 0.614 & 0.572 \\
\bottomrule
\end{tabular}
\end{table*}

\section{Results and Analysis}

Our analysis primarily examines how replacing clipping with a convex quadratic penalty affects optimization behavior and stability, and how these differences translate into final downstream performance in LLM post-training. We study this across reasoning-focused reinforcement learning (RLVR) and alignment-focused reinforcement learning (RLHF), and under multiple sources of off-policy pressure.

\subsection{Reasoning}

\paragraph{Cold-Start Training with Qwen Models.}

Following recent work demonstrating that RL can be applied directly to base language models without prior supervised fine-tuning~\citep{deepseekai2025deepseekr1incentivizingreasoningcapability,marjanović2026deepseekr1thoughtologyletsthink}, we perform cold-start training using both GRPO and CFPO across multiple Qwen2.5 models, scaling the number of training iterations up to 16.

In general, we expect GRPO to optimize reward more aggressively than CFPO: GRPO performs direct advantage maximization subject to a clipped constraint, whereas CFPO's quadratic penalty continues to regularize updates even when probability ratios lie within the trust region. We also expect CFPO to exhibit lower clipping ratios, since it actively discourages large updates rather than discarding clipped samples. 

Figure~\ref{fig:model_metrics_comparison-trl} confirms these expectations. GRPO improves reward more rapidly in early training, while CFPO progresses more gradually and eventually reaches comparable levels. Consistent with its design, CFPO maintains lower clipping ratios throughout. However, GRPO becomes increasingly unstable as iterations grow: instability appears around 8 iterations, and training collapses by 16 across most configurations.

We evaluate downstream reasoning performance on Math500, GSM8K, AIME24, and GPQA-Diamond (Tables~\ref{tab:qwen1.5b-base}, \ref{tab:qwen3b-base}, \ref{tab:qwen7b-base}). On Math500, GRPO and CFPO perform comparably across model sizes. For GRPO, collapse occurs at 8 iterations for all models except Qwen2.5-7B, where it is delayed until 16. CFPO exhibits breakdown at 16 iterations for the 1.5B and 7B models, while remaining always stable for the 3B model.

Interestingly, CFPO underperforms on GSM8K compared to GRPO for the 1.5B and 7B models. Qualitative inspection suggests this gap reflects weaker instruction-following rather than degraded reasoning: CFPO-trained models more frequently produce incomplete generations or responses in unintended languages. We corroborate this using AlpacaEval, where CFPO-trained models achieve lower scores. On AIME24 and GPQA-Diamond, we find no evidence that GRPO-trained models consistently outperform CFPO, suggesting GRPO's GSM8K advantage is attributable to more aggressive instruction learning rather than better reasoning.

Overall, these results reveal a trade-off between optimization aggressiveness and stability. GRPO achieves faster early reward gains but exhibits growing instability at higher iteration counts. CFPO advances more gradually, maintaining better stability while reaching comparable downstream reasoning performance.

\begin{figure*}[t]
    \centering
    \includegraphics[width=\textwidth]{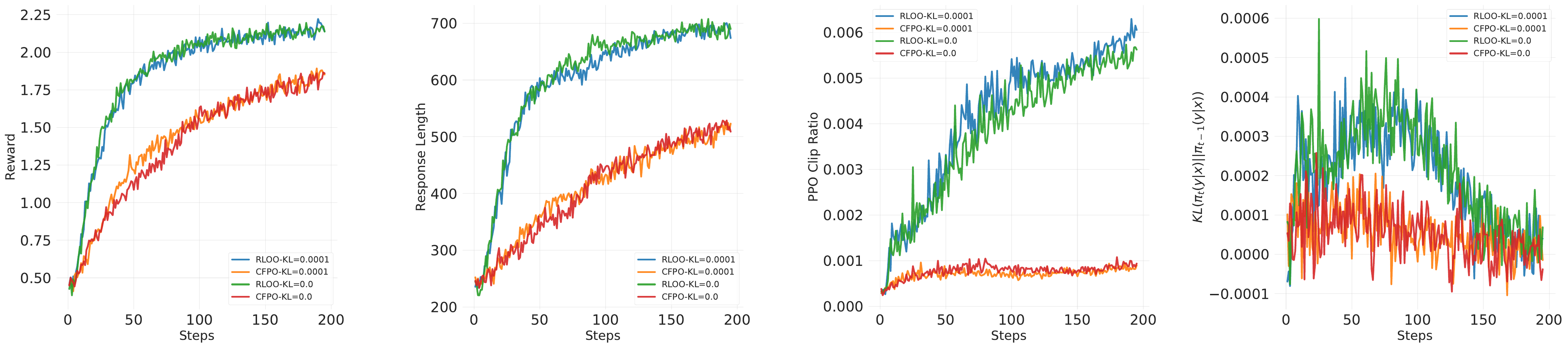}
    \caption{RLHF training dynamics on Llama3-8B under RLOO and CFPO with different KL penalty coefficients. We report trajectories over training steps for (a) reward, (b) generated response length, (c) policy clipping ratio, and (d) KL divergence between consecutive policy updates. RLOO exhibits rapid early reward increases accompanied by growing response lengths and elevated clipping activity, particularly when the KL penalty is weak or removed, indicating more aggressive optimization. In contrast, CFPO yields steadier reward improvement while maintaining stable response lengths, lower clipping ratios, and controlled KL divergence across settings, reflecting more conservative and stable policy updates during RLHF.}
    \label{fig:rlhf-training}
\end{figure*}

\paragraph{RLVR on Instruction-Tuned Models.}

To separate effects arising from instruction-following behavior from those related to reasoning, we perform RLVR on instruction-tuned versions of the same models. Consistent with cold-start results, we observe no substantial performance differences between GRPO and CFPO across model scales (Tables~\ref{tab:qwen1.5b}, \ref{tab:qwen3b}, \ref{tab:qwen7b}). The main difference remains optimization stability: GRPO exhibits instability at higher iteration counts, while CFPO remains stable across model sizes and training durations which is consistent with cold-start experiments.

\paragraph{Off-Policy Training with \texttt{verl}.}

Several RLVR recipes implemented in \texttt{verl} have non-zero clipping ratios even without explicit sample reuse. Inspection of these configurations reveals that off-policy effects arise primarily from mini-batch policy updates, which introduce mild policy lag despite using fresh data. This regime has motivated the use of relatively large clipping thresholds in prior work~\citep{yu2025dapoopensourcellmreinforcement,yang2025dcpodynamicclippingpolicy}.

At a single iteration, observations closely mirror those from TRL. GRPO optimizes reward rapidly, while SPO converges more slowly but eventually catches up, achieving comparable training reward, validation reward, policy clipping fractions, and KL (Figs.~\ref{fig:verl-trainreward}, \ref{fig:verl-valreward}, \ref{fig:verl-pg_clipfrac}, \ref{fig:verl-ppo-kl}). In contrast, entropy dynamics differ markedly: GRPO consumes entropy substantially faster than SPO, reflecting more aggressive optimization behavior (Fig.~\ref{fig:verl-entropy}). While this accelerates short-horizon reward gains, it likely undermines stability in longer training regimes where sustained exploration is required.

We next examine how different sources of off-policy pressure affect stability. Increasing the batch ratio up to 8 does not, by itself, induce collapse, whereas increasing the number of iterations leads to instability around 8 iterations (Figs.~\ref{fig:verl-trainreward}, \ref{fig:verl-pg_clipfrac}). This contrast suggests that batch-ratio–induced off-policy updates with fresh data are less destabilizing than iteration-based sample reuse, which more readily amplifies policy lag.

However, these effects are not independent. Even at iteration counts where GRPO remains stable (e.g., 4 iterations), increasing the batch ratio can still induce collapse, indicating that off-policy pressure accumulates across mechanisms (Figs.~\ref{fig:verl-pg_clipfrac}, \ref{fig:verl-entropy}). In contrast, SPO consistently maintains lower clipping ratios, slower entropy decay, and more stable policy updates across configurations, contributing to improved robustness under off-policy regimes.

\paragraph{Extensibility.}

Moreover, recent RLVR work has introduced orthogonal techniques including token-level loss aggregation, dynamic sampling, and entropy regularization~\citep{yu2025dapoopensourcellmreinforcement,liu2025understandingr1zeroliketrainingcritical}. Since CFPO modifies only the surrogate objective, these techniques transfer directly; practitioners can replace clipped objectives with CFPO's penalty without further modification.

\subsection{RLHF}
We train CFPO in standard RLHF using RLOO advantage estimation on LLaMA-3-8B. All hyperparameters are taken directly from the OpenRLHF repository without tuning; the only modification is replacing the clipped surrogate with CFPO’s quadratic penalty. This isolates the effect of the objective itself. Tables~\ref{tab:alignment} and \ref{tab:capability} report results, with training dynamics shown in Figure~\ref{fig:rlhf-training}.

\paragraph{Optimization Dynamics.}

As in the reasoning setting, we first examine how the choice of objective shapes optimization behavior during training. Figure~\ref{fig:rlhf-training} shows that RLOO exhibits aggressive early optimization—rewards increase rapidly before plateauing—while CFPO demonstrates more gradual, approximately linear reward improvement throughout training. This pattern extends to other metrics: under RLOO, response length increases alongside reward, indicating the model learns to exploit verbosity as a reward-hacking strategy, while CFPO maintains stable response lengths. RLOO's clipping ratios also increase over training while CFPO keeps these values lower and more consistent.

\paragraph{Instruction Following.}

We next evaluate how these differing optimization dynamics translate into instruction-following behavior. Both methods achieve competitive scores on Arena-Hard and MT-Bench. The length exploitation observed during training becomes apparent in AlpacaEval: while raw win rates appear similar, length-controlled scores reveal a meaningful gap, with CFPO outperforming RLOO by 3--4 percentage points. This confirms that RLOO inflates its win rate through verbosity, while CFPO's consistent scores between raw and length-controlled metrics indicate genuine quality improvements.

Furthermore, RLOO substantially degrades performance on IFEval, dropping from 59.6 to 47.0 compared to the SFT baseline—a 12-point decline in exact instruction-following capability. CFPO preserves this capability much better, achieving 55.6 with only a 4-point drop. This suggests that RLOO's aggressive optimization not only exploits superficial reward correlates but limits the model's ability to follow precise instructions, while CFPO's conservative updates maintain this capability.

\paragraph{Capability Retention.}

Finally, we assess how these optimization differences affect retention of base model capabilities. RLHF methods often degrade base model capabilities, a phenomenon known as alignment tax~\citep{askell2021generallanguageassistantlaboratory}. RLOO incurs 12--16\% alignment tax depending on KL penalty settings, while CFPO variants pay only 4--5\%. The degradation under RLOO spans ARC, HellaSwag, and Winogrande, whereas CFPO retains substantially more capability across all evaluations.

These results mirror our reasoning findings: across both settings, CFPO's quadratic penalty produces more stable optimization dynamics and competitive downstream performance. The conservative updates that prevent training collapse in reasoning also prevent the excessive policy drift that degrades capabilities in RLHF.

\section{Related Work}

\paragraph{Stable Policy Gradient Methods.}
Trust region methods have been fundamental to stable policy optimization since the natural policy gradient \citep{NIPS2001_4b86abe4} and TRPO \citep{schulman2017trustregionpolicyoptimization} established theoretical guarantees for monotonic improvement. PPO \citep{schulman2017proximalpolicyoptimizationalgorithms} made these methods practical through clipping, becoming dominant for both continuous control and language model training. However, \citet{wang2020trulyproximalpolicyoptimization} showed that clipping fails to bound KL divergence, while \citet{engstrom2020implementationmattersdeeppolicy} demonstrated that implementation details matter more than clipping itself. These findings have motivated principled alternatives: f-divergence generalizations \citep{belousov2018fdivergenceconstrainedpolicyimprovement}, mirror descent interpretations \citep{tomar2021mirrordescentpolicyoptimization,lan2022policymirrordescentreinforcement}, and algorithms like MPO \citep{abdolmaleki2018maximumposterioripolicyoptimisation} and AWR \citep{peng2019advantageweightedregressionsimplescalable} that enforce trust regions through different mechanisms. SPO \citep{xie2025simplepolicyoptimization} replaces clipping with a convex quadratic penalty derived from TV divergence, providing non-zero gradients while implicitly enforcing trust region bounds. We extend SPO to language model training.

\paragraph{Reinforcement Learning for Language Models.}
In RLHF, reward models trained on human preferences guide policy optimization \citep{christiano2023deepreinforcementlearninghuman,ziegler2020finetuninglanguagemodelshuman,stiennon2022learningsummarizehumanfeedback,bai2022constitutionalaiharmlessnessai,ouyang2022traininglanguagemodelsfollow}, though challenges like reward overoptimization \citep{gao2022scalinglawsrewardmodel} have motivated simpler approaches. DPO \citep{rafailov2024directpreferenceoptimizationlanguage} bypasses reward models entirely, spawning variants including IPO \citep{azar2023generaltheoreticalparadigmunderstand}, KTO \citep{ethayarajh2024ktomodelalignmentprospect}, and SimPO \citep{meng2024simposimplepreferenceoptimization}. Among online methods, RLOO \citep{ahmadian2024basicsrevisitingreinforcestyle} and ReMax \citep{li2024remaxsimpleeffectiveefficient} eliminate the critic while retaining on-policy benefits. 

For reasoning, RL with verifiable rewards has proven effective \citep{openai2024openaio1card,deepseekai2025deepseekr1incentivizingreasoningcapability}, building on work showing process supervision improves reasoning \citep{lightman2023letsverifystepstep,zelikman2022starbootstrappingreasoningreasoning}. GRPO \citep{shao2024deepseekmath} is now standard, using group-relative advantages without a critic. However, scaling has revealed clipping-related instabilities, motivating variants like DAPO \citep{yu2025dapoopensourcellmreinforcement}, Dr.GRPO \citep{liu2025understandingr1zeroliketrainingcritical}, and $\lambda$-GRPO \citep{wang2025lambdagrpounifyinggrpoframeworks} that modify clipping behavior or token weighting. Rather than patching clipping, CFPO replaces it entirely with the quadratic penalty, providing a unified approach across both RLHF and reasoning settings.

\section{Discussion and Future Work}
Our experiments span models from 1.5B to 8B parameters across reasoning and alignment tasks, but are limited to Qwen and LLaMA model families trained on MATH and OpenRLHF datasets. Frontier models operate at substantially larger scales with longer training horizons; we did not have sufficient compute to evaluate CFPO in these regimes, and whether its conservative updates remain beneficial or become overly restrictive at scale is an open question. Similarly, due to resource constraints, we could not explore greater diversity in model architectures, training datasets, and domains with sparser or noisier rewards such as code generation or agentic applications which may reveal different optimization dynamics. We leave these directions to future work.

\section{Conclusion}

We propose CFPO as a drop-in replacement for the clipped surrogate objectives used in PPO and GRPO in language model post-training. By replacing heuristic clipping with a convex quadratic penalty derived from Total Variation divergence constraints, CFPO provides smooth gradients throughout the optimization landscape while implicitly enforcing trust region bounds. Our experiments across both reasoning and alignment settings demonstrate that CFPO offers improved training stability—substantially delaying collapse at high iteration counts, reducing alignment tax, and maintaining gradual entropy consumption—while achieving competitive downstream performance, all at no additional computational cost and with only a one-line code change. These findings suggest that clipping's aggressive optimization dynamics, while effective at rapidly acquiring surface-level patterns, introduces instabilities problematic at scale, and that CFPO's more conservative updates offer a promising alternative for language model post-training.

\section*{Acknowledgements}

The authors gratefully acknowledge the scientific support and HPC resources provided by the Erlangen National High Performance Computing Center (NHR@FAU) of the Friedrich-Alexander-Universität Erlangen-Nürnberg (FAU). The hardware is funded by the German Research Foundation (DFG).



\bibliography{example_paper}
\bibliographystyle{icml2026}

\newpage
\appendix
\onecolumn
\section{Theoretical Results of Simple Policy Optimization}
\label{app:theory}
In this appendix, we restate the key theoretical results from \citet{xie2025simplepolicyoptimization} that motivate Simple Policy Optimization (SPO) as an alternative to PPO's clipping mechanism.

\subsection{Performance Improvement Bounds}

The foundation of trust region methods lies in the policy performance difference theorem \citep{Kakade2002ApproximatelyOA}, which expresses the performance gap between policies in terms of advantage functions. Building on this, \citet{achiam2017constrainedpolicyoptimization} established a performance improvement lower bound using Total Variation (TV) divergence:

\begin{theorem}[Performance Improvement Lower Bound]
\label{thm:tv_bound}
Given any two policies $\pi$ and $\tilde{\pi}$, the following bound holds:
\begin{equation}
    \eta(\tilde{\pi}) - \eta(\pi) \geq \frac{1}{1-\gamma}\mathbb{E}_{s \sim \rho_\pi, a \sim \tilde{\pi}}\left[A_\pi(s,a)\right] - \frac{2\xi\gamma}{(1-\gamma)^2}\mathbb{E}_{s \sim \rho_\pi}\left[D_{\mathrm{TV}}(\pi \| \tilde{\pi})[s]\right],
\end{equation}
where $\xi = \max_s \left|\mathbb{E}_{a \sim \tilde{\pi}(\cdot|s)}\left[A_\pi(s,a)\right]\right|$ and $D_{\mathrm{TV}}$ denotes the Total Variation divergence.
\end{theorem}

Using the relationship between TV divergence and probability ratios, this bound can be rewritten as:
\begin{equation}
\label{eq:ratio_bound}
    \eta(\tilde{\pi}) - \eta(\pi) \geq \frac{1}{1-\gamma}\mathbb{E}_{s,a \sim \pi}\left[\frac{\tilde{\pi}(a|s)}{\pi(a|s)} A_\pi(s,a)\right] - \frac{\xi\gamma}{(1-\gamma)^2}\mathbb{E}_{s,a \sim \pi}\left[\left|\frac{\tilde{\pi}(a|s)}{\pi(a|s)} - 1\right|\right].
\end{equation}

This formulation reveals why constraining the probability ratio $|r_t - 1| \leq \epsilon$ is beneficial: it directly controls the penalty term in the performance bound.

\subsection{Advantages of TV Divergence over KL Divergence}

A key insight from \citet{xie2025simplepolicyoptimization} is that TV divergence constraints offer theoretical advantages over the KL divergence constraints used in TRPO.

\begin{proposition}[TV Solution Space Contains KL Solution Space]
\label{prop:tv_contains_kl}
Given an old policy $\pi$, define the solution spaces under TV and KL divergence constraints as:
\begin{align}
    \Omega_{\mathrm{TV}} &= \{\tilde{\pi} \mid D_{\mathrm{TV}}(\pi \| \tilde{\pi})[s] \leq \delta_{\mathrm{TV}}, \forall s \in \mathcal{S}\}, \\
    \Omega_{\mathrm{KL}} &= \{\tilde{\pi} \mid D_{\mathrm{KL}}(\pi \| \tilde{\pi})[s] \leq \delta_{\mathrm{KL}}, \forall s \in \mathcal{S}\},
\end{align}
where $\delta_{\mathrm{KL}} > 0$ is a predefined threshold. For $\delta_{\mathrm{TV}} \geq \sqrt{\frac{1}{2}\delta_{\mathrm{KL}}}$, we have $\Omega_{\mathrm{KL}} \subset \Omega_{\mathrm{TV}}$.
\end{proposition}

\begin{proof}
For any $\tilde{\pi} \in \Omega_{\mathrm{KL}}$, by Pinsker's inequality:
\begin{equation}
    D_{\mathrm{TV}}(\pi \| \tilde{\pi})[s] \leq \sqrt{\frac{1}{2}D_{\mathrm{KL}}(\pi \| \tilde{\pi})[s]} \leq \sqrt{\frac{1}{2}\delta_{\mathrm{KL}}} \leq \delta_{\mathrm{TV}}.
\end{equation}
Thus $\tilde{\pi} \in \Omega_{\mathrm{KL}} \Rightarrow \tilde{\pi} \in \Omega_{\mathrm{TV}}$, establishing $\Omega_{\mathrm{KL}} \subset \Omega_{\mathrm{TV}}$.
\end{proof}

Furthermore, optimizing within the larger TV-constrained space yields superior bounds:

\begin{theorem}[Superiority of TV-Constrained Optimization]
\label{thm:tv_superior}
Let $\mathcal{L}_\pi^{\mathrm{TV}}(\tilde{\pi})$ and $\mathcal{L}_\pi^{\mathrm{KL}}(\tilde{\pi})$ denote the performance improvement lower bounds with TV and KL divergence penalties respectively:
\begin{align}
    \mathcal{L}_\pi^{\mathrm{TV}}(\tilde{\pi}) &= \frac{1}{1-\gamma}\mathbb{E}_{s,a \sim \pi}\left[\frac{\tilde{\pi}(a|s)}{\pi(a|s)} A_\pi(s,a)\right] - \frac{2\xi\gamma}{(1-\gamma)^2}\mathbb{E}_{s \sim \rho_\pi}\left[D_{\mathrm{TV}}(\pi \| \tilde{\pi})[s]\right], \\
    \mathcal{L}_\pi^{\mathrm{KL}}(\tilde{\pi}) &= \frac{1}{1-\gamma}\mathbb{E}_{s,a \sim \pi}\left[\frac{\tilde{\pi}(a|s)}{\pi(a|s)} A_\pi(s,a)\right] - \frac{2\xi\gamma}{(1-\gamma)^2}\mathbb{E}_{s \sim \rho_\pi}\left[\sqrt{\frac{1}{2}D_{\mathrm{KL}}(\pi \| \tilde{\pi})[s]}\right].
\end{align}
Define the optimal policies in each constrained space:
\begin{equation}
    \tilde{\pi}_{\mathrm{TV}}^* = \arg\max_{\tilde{\pi} \in \Omega_{\mathrm{TV}}} \mathcal{L}_\pi^{\mathrm{TV}}(\tilde{\pi}), \quad \tilde{\pi}_{\mathrm{KL}}^* = \arg\max_{\tilde{\pi} \in \Omega_{\mathrm{KL}}} \mathcal{L}_\pi^{\mathrm{KL}}(\tilde{\pi}).
\end{equation}
For $\delta_{\mathrm{TV}} \geq \sqrt{\frac{1}{2}\delta_{\mathrm{KL}}}$, we have $\mathcal{L}_\pi^{\mathrm{TV}}(\tilde{\pi}_{\mathrm{TV}}^*) \geq \mathcal{L}_\pi^{\mathrm{KL}}(\tilde{\pi}_{\mathrm{KL}}^*)$.
\end{theorem}

\begin{proof}
Since $\Omega_{\mathrm{KL}} \subset \Omega_{\mathrm{TV}}$ by Proposition~\ref{prop:tv_contains_kl}:
\begin{align}
    \mathcal{L}_\pi^{\mathrm{TV}}(\tilde{\pi}_{\mathrm{TV}}^*) &\geq \mathcal{L}_\pi^{\mathrm{TV}}(\tilde{\pi}_{\mathrm{KL}}^*) \\
    &= \frac{1}{1-\gamma}\mathbb{E}_{s,a \sim \pi}\left[\frac{\tilde{\pi}_{\mathrm{KL}}^*(a|s)}{\pi(a|s)} A_\pi(s,a)\right] - \frac{2\xi\gamma}{(1-\gamma)^2}\mathbb{E}_{s \sim \rho_\pi}\left[D_{\mathrm{TV}}(\pi \| \tilde{\pi}_{\mathrm{KL}}^*)[s]\right] \\
    &\geq \frac{1}{1-\gamma}\mathbb{E}_{s,a \sim \pi}\left[\frac{\tilde{\pi}_{\mathrm{KL}}^*(a|s)}{\pi(a|s)} A_\pi(s,a)\right] - \frac{2\xi\gamma}{(1-\gamma)^2}\mathbb{E}_{s \sim \rho_\pi}\left[\sqrt{\frac{1}{2}D_{\mathrm{KL}}(\pi \| \tilde{\pi}_{\mathrm{KL}}^*)[s]}\right] \\
    &= \mathcal{L}_\pi^{\mathrm{KL}}(\tilde{\pi}_{\mathrm{KL}}^*),
\end{align}
where the second inequality follows from Pinsker's inequality.
\end{proof}

\subsection{The $\epsilon$-Aligned Objective Class}

To understand why PPO's clipping fails to constrain probability ratios while SPO succeeds, \citet{xie2025simplepolicyoptimization} introduce the concept of $\epsilon$-aligned objectives.

Based on the performance bound in Equation~\eqref{eq:ratio_bound}, the goal is to solve the following constrained optimization problem:
\begin{equation}
\label{eq:constrained_opt}
\begin{split}
    \max_{\theta} \enspace & \mathbb{E}_{(s_t, a_t) \sim \pi_{\theta_{\mathrm{old}}}} \left[ r_t(\theta) \cdot \hat{A}_t \right], \\
    \text{s.t.} \enspace & \mathbb{E}_{(s_t, a_t) \sim \pi_{\theta_{\mathrm{old}}}} \left[ |r_t(\theta) - 1| \right] \leq \epsilon,
\end{split}
\end{equation}
where $r_t(\theta) = \pi_\theta(a_t|s_t) / \pi_{\theta_{\mathrm{old}}}(a_t|s_t)$ and $\hat{A}_t = \hat{A}(s_t, a_t)$.

For a single data point with advantage $A \neq 0$, this simplifies to:
\begin{equation}
    \max_r \; rA, \quad \text{s.t.} \; |r - 1| \leq \epsilon.
\end{equation}
Since the objective is linear in $r$, the optimal solution lies at the constraint boundary: $r^* = 1 + \mathrm{sign}(A) \cdot \epsilon$.

\begin{definition}[$\epsilon$-Aligned Objective]
\label{def:epsilon_aligned}
For any given $A \neq 0$ and $\epsilon > 0$, a function $f(r, A, \epsilon)$ is \emph{$\epsilon$-aligned} if it is differentiable and convex with respect to $r$, and attains its maximum at $r = 1 + \mathrm{sign}(A) \cdot \epsilon$.
\end{definition}

An $\epsilon$-aligned objective converts the constrained problem into an unconstrained one whose optimal solution automatically satisfies the constraint. The PPO and SPO objectives can be expressed as:
\begin{align}
    f_{\mathrm{PPO}} &= \min\left[rA, \mathrm{clip}(r, 1-\epsilon, 1+\epsilon) \cdot A\right], \\
    f_{\mathrm{SPO}} &= rA - \frac{|A|}{2\epsilon}(r-1)^2.
\end{align}

\begin{theorem}[SPO is $\epsilon$-Aligned]
\label{thm:spo_aligned}
The SPO objective $f_{\mathrm{SPO}}$ is $\epsilon$-aligned, while the PPO objective $f_{\mathrm{PPO}}$ is not.
\end{theorem}

\begin{proof}
For $f_{\mathrm{SPO}}$: The objective is a quadratic polynomial in $r$, hence differentiable and convex. Setting the derivative to zero:
\begin{equation}
    \frac{\partial f_{\mathrm{SPO}}}{\partial r} = A - \frac{|A|}{\epsilon}(r-1) = 0 \implies r = 1 + \mathrm{sign}(A) \cdot \epsilon.
\end{equation}
Thus $f_{\mathrm{SPO}}$ is $\epsilon$-aligned.

For $f_{\mathrm{PPO}}$: The clipping operation causes the gradient to vanish when $r > 1 + \epsilon$ and $A > 0$, or when $r < 1 - \epsilon$ and $A < 0$. This means $f_{\mathrm{PPO}}$ is not differentiable everywhere and does not satisfy the $\epsilon$-aligned definition.
\end{proof}

The practical consequence is that PPO's clipped objective zeros gradients for data points outside the trust region, providing no corrective signal to bring them back. In contrast, SPO maintains non-zero gradients that consistently guide optimization toward the constraint boundary $r = 1 + \mathrm{sign}(A) \cdot \epsilon$, ensuring effective probability ratio control throughout training.
\newpage
\section{Training Details}
\label{app:training}

\subsection{Reasoning Hyperparameters}
Training hyperparameters for reasoning experiments are listed in Table~\ref{tab:reasoning-hparams}. The same hyperparameters are used for both base and instruct model variants.

\begin{table}[ht]
\centering
\caption{Training hyperparameters for Qwen2.5 reasoning experiments. Only hyperparameters that affect the learned policy or evaluation are listed. Unspecified fields inherit the TRL\_v0.8 defaults.}
\setlength{\tabcolsep}{6pt}
\begin{tabular}{lcc}
\toprule
Parameter & 1.5B/3B & 7B\\ 
\midrule
Learning Rate           & $3\times10^{-6}$ & $1\times10^{-6}$\\
Batch Size              & 128 & 128\\
Group Size              & 3 & 3\\
KL Penalty ($\beta$)    & 0 & 0\\
Max Prompt Length       & 512 & 512\\
Max Completion Length   & 3072 & 3072\\
Temperature             & 0.9 & 0.9\\
Clip Ratio              & 0.2 & 0.2\\
Lr Scheduler Type       & Cosine & Cosine\\
Warmup Ratio            & 0.1 & 0.1\\
Optimizer               & \multicolumn{2}{c}{AdamW $(\beta_1{=}0.9,\ \beta_2{=}0.999,\ \varepsilon{=}10^{-8})$}\\
\bottomrule
\end{tabular}
\label{tab:reasoning-hparams}
\end{table}

The system prompts used during training are shown below. The same prompts are used for both base and instruct model variants.

\tcbset{
  mybox/.style={
    colback=white,
    colframe=gray!50!black,
    fonttitle=\bfseries,
    boxrule=0.5pt,
    arc=2pt,
    breakable,
    listing only,
    listing options={basicstyle=\ttfamily\small,
    breaklines=true,
    inputencoding=utf8}
  }
}

\begin{tcolorbox}[title={System prompt used for Qwen2.5-1.5B models.}, mybox, breakable]
You are a helpful AI Assistant, designed to provided well-reasoned and
detailed responses. You FIRST think about the reasoning process step by
step and then provide the user with the answer. Please enclose your final
answer in the box: \textbackslash boxed\{Your Answer\}.
\end{tcolorbox}

\begin{tcolorbox}[title={System prompt used for Qwen2.5-3B models.}, mybox, breakable]
You are a helpful AI Assistant, designed to provided well-reasoned and
detailed responses. You FIRST think about the reasoning process step by
step and then provide the user with the answer. Please enclose your final
answer in the box: \textbackslash boxed\{Your Answer\}. Please stop generation immediately
after outputing the box.
\end{tcolorbox}

\begin{tcolorbox}[title={System prompt used for Qwen2.5-7B models.}, mybox, breakable]
You are a helpful AI Assistant, designed to provided well-reasoned and detailed responses. Please provide a step-by-step solution to the following problem.
\end{tcolorbox}

\subsection{RLHF Hyperparameters}
For RLHF experiments on Llama-3-8B, we use the OpenRLHF framework with default hyperparameters, only modifying the policy loss to use CFPO. Key hyperparameters are listed in Table~\ref{tab:rlhf-hparams}. We use Llama-3-8B-SFT-Mixture as the base model and Llama-3-8B-RM-700K as the reward model, both from OpenRLHF. Training is conducted on the OpenRLHF prompt-collection-v0.1 dataset.

\begin{table}[h]
\centering
\caption{RLHF training hyperparameters for Llama-3-8B experiments.}
\setlength{\tabcolsep}{6pt}
\begin{tabular}{lc}
\toprule
Parameter & Value\\ 
\midrule
Actor Learning Rate     & $5\times10^{-7}$\\
Train Batch Size        & 64\\
Rollout Batch Size      & 512\\
Samples per Prompt      & 2\\
Max Prompt Length       & 1024\\
Max Generation Length   & 1024\\
KL Penalty ($\beta$)    & $1\times10^{-4}$ / $0\times10^{-4}$ \\
Advantage Estimator     & RLOO\\
\bottomrule
\end{tabular}
\label{tab:rlhf-hparams}
\end{table}

\newpage

\section{Qwen2.5 Results in TRL}

\begin{table}[htbp]
\centering
\caption{RLVR performance of Qwen2.5-1.5B-Instruct on MATH500, GSM8K, GPQA-Diamond, and AIME24 under GRPO and CFPO across increasing iteration counts. Both methods refine existing instruction-following and reasoning behavior with limited gains under small training budgets. The primary difference emerges in optimization stability: GRPO degrades at higher iteration counts, while CFPO maintains more consistent performance across tasks.}
\label{tab:qwen1.5b}
\begin{tabular}{lcccc}
\toprule
\textbf{Model} & \textbf{Math 500} & \textbf{GSM8K} & \textbf{GPQA-Diamond} & \textbf{AIME 24}\\
\midrule
Qwen2.5-1.5B-Instruct & 0.528 & 0.648 & 0.272 & 0.033 \\
\midrule
\multicolumn{5}{l}{\textit{GRPO Training}} \\
\quad + GRPO-iteration 2  & 0.580 & 0.745 & 0.222 & 0.033 \\
\quad + GRPO-iteration 4  & 0.544 & 0.682 & 0.277 & 0.033 \\
\quad + GRPO-iteration 8  & 0.370 & 0.604 & 0.248 & 0.033 \\
\quad + GRPO-iteration 16 & 0.100 & 0.349 & 0.262 & 0.000 \\
\midrule
\multicolumn{5}{l}{\textit{CFPO Training}} \\
\quad + CFPO-iteration 2  & 0.544 & 0.726 & 0.242 & 0.033 \\
\quad + CFPO-iteration 4  & 0.534 & 0.692 & 0.333 & 0.033 \\
\quad + CFPO-iteration 8  & 0.524 & 0.714 & 0.282 & 0.000 \\
\quad + CFPO-iteration 16 & 0.500 & 0.689 & 0.248 & 0.000 \\
\bottomrule
\end{tabular}
\end{table}

\begin{table}[htbp]
\centering
\caption{RLVR results for Qwen2.5-3B-Instruct across multiple reasoning benchmarks and iteration counts. 
With instruction-following already present, both GRPO and CFPO yield similar downstream behavior, indicating limited scope for qualitative improvement under short-horizon RL. 
Nevertheless, CFPO remains stable across all iteration counts, whereas GRPO exhibits degradation as off-policy pressure increases.}
\label{tab:qwen3b}
\begin{tabular}{lcccc}
\toprule
\textbf{Model} & \textbf{Math 500} & \textbf{GSM8K} & \textbf{GPQA-Diamond} & \textbf{AIME 24} \\
\midrule
Qwen2.5-3B-Instruct & 0.645 & 0.764 & 0.328 & 0.03 \\
\midrule
\multicolumn{5}{l}{\textit{GRPO Training}} \\
\quad + GRPO-iteration 2  & 0.656 & 0.793 & 0.292 & 0.067 \\
\quad + GRPO-iteration 4  & 0.612 & 0.785 & 0.323 & 0.000 \\
\quad + GRPO-iteration 8  & 0.450 & 0.759 & 0.262 & 0.033 \\
\quad + GRPO-iteration 16 & 0.080 & 0.205 & 0.258 & 0.067 \\
\midrule
\multicolumn{5}{l}{\textit{CFPO Training}} \\
\quad + CFPO-iteration 2  & 0.658 & 0.762 & 0.333 & 0.033 \\
\quad + CFPO-iteration 4  & 0.674 & 0.796 & 0.288 & 0.067 \\
\quad + CFPO-iteration 8  & 0.660 & 0.769 & 0.343 & 0.100 \\
\quad + CFPO-iteration 16 & 0.558 & 0.752 & 0.328 & 0.100 \\
\bottomrule
\end{tabular}
\end{table}

\begin{table}[htbp]
\centering
\caption{RLVR performance of Qwen2.5-7B-Instruct under GRPO and CFPO across iteration counts. At this scale, both methods achieve comparable refinement of reasoning and instruction-following behavior. However, GRPO becomes unstable at high iteration counts, while CFPO preserves consistent performance, highlighting its robustness under extended training.}
\label{tab:qwen7b}
\begin{tabular}{lcccc}
\toprule
\textbf{Model} & \textbf{Math 500} & \textbf{GSM8K} & \textbf{GPQA-Diamond} & \textbf{AIME 24} \\
\midrule
Qwen2.5-7B-Instruct & 0.756 & 0.816 & 0.363 & 0.100 \\
\midrule
\multicolumn{5}{l}{\textit{GRPO Training}} \\
\quad + GRPO-iteration 2  & 0.750 & 0.893 & 0.369 & 0.067 \\
\quad + GRPO-iteration 4  & 0.774 & 0.906 & 0.354 & 0.133 \\
\quad + GRPO-iteration 8  & 0.764 & 0.879 & 0.394 & 0.133 \\
\quad + GRPO-iteration 16 & 0.080 & 0.316 & 0.277 & 0.000 \\
\midrule
\multicolumn{5}{l}{\textit{CFPO Training}} \\
\quad + CFPO-iteration 2  & 0.758 & 0.845 & 0.348 & 0.100 \\
\quad + CFPO-iteration 4  & 0.772 & 0.839 & 0.338 & 0.167 \\
\quad + CFPO-iteration 8  & 0.768 & 0.835 & 0.318 & 0.100 \\
\quad + CFPO-iteration 16 & 0.758 & 0.869 & 0.318 & 0.130 \\
\bottomrule
\end{tabular}
\end{table}

\begin{table}[htbp]
\centering
\caption{Cold-start RL performance of Qwen2.5-1.5B on reasoning benchmarks under GRPO and CFPO across iteration counts. 
Both methods enable the emergence of reasoning behavior from a base model without prior instruction tuning. 
GRPO exhibits rapid early improvement followed by instability as iterations increase, whereas CFPO progresses more conservatively and maintains stable performance over longer training horizons.}
\label{tab:qwen1.5b-base}
\begin{tabular}{lcccc}
\toprule
\textbf{Model} & \textbf{Math 500} & \textbf{GSM8K} & \textbf{GPQA-Diamond} & \textbf{AIME 24} \\
\midrule
Qwen2.5-1.5B & 0.002 & 0.090 & 0.171 & 0 \\
\midrule
\multicolumn{5}{l}{\textit{GRPO Training}} \\
\quad GRPO-Iteration2  & 0.506 & 0.672 & 0.242 & 0.00 \\
\quad GRPO-Iteration4  & 0.558 & 0.729 & 0.293 & 0.00 \\
\quad GRPO-Iteration8  & 0.120 & 0.030 & 0.278 & 0.00 \\
\quad GRPO-Iteration16 & 0.074 & 0.039 & 0.197 & 0.00 \\
\midrule
\multicolumn{5}{l}{\textit{CFPO Training}} \\
\quad CFPO-Iteration2  & 0.514 & 0.548 & 0.288 & 0.067 \\
\quad CFPO-Iteration4  & 0.574 & 0.570 & 0.253 & 0.033 \\
\quad CFPO-Iteration8  & 0.530 & 0.280 & 0.253 & 0.033 \\
\quad CFPO-Iteration16 & 0.380 & 0.020 & 0.308 & 0.033 \\
\bottomrule
\end{tabular}
\end{table}

\begin{table}[htbp]
\centering
\caption{Cold-start RL results for Qwen2.5-3B across reasoning benchmarks and iteration counts. This model scale shows the clearest stability advantage for CFPO, which remains robust across all iteration settings. In contrast, GRPO degrades as iteration counts grow, consistent with the training dynamics observed in reward and clipping metrics.}
\label{tab:qwen3b-base}
\begin{tabular}{lcccc}
\toprule
\textbf{Model} & \textbf{Math 500} & \textbf{GSM8K} & \textbf{GPQA-Diamond} & \textbf{AIME 24} \\
\midrule
Qwen2.5-3B & 0.673 & 0.544 & 0.262 & 0.06 \\
\midrule
\multicolumn{5}{l}{\textit{GRPO Training}} \\
\quad GRPO-Iteration2  & 0.642 & 0.825 & 0.318 & 0.067 \\
\quad GRPO-Iteration4  & 0.648 & 0.796 & 0.303 & 0.033 \\
\quad GRPO-Iteration8  & 0.462 & 0.669 & 0.248 & 0.033 \\
\quad GRPO-Iteration16 & 0.220 & 0.673 & 0.252 & 0.033 \\
\midrule
\multicolumn{5}{l}{\textit{CFPO Training}} \\
\quad CFPO-Iteration2  & 0.634 & 0.766 & 0.313 & 0.100 \\
\quad CFPO-Iteration4  & 0.616 & 0.767 & 0.328 & 0.067 \\
\quad CFPO-Iteration8  & 0.594 & 0.779 & 0.268 & 0.067 \\
\quad CFPO-Iteration16 & 0.604 & 0.798 & 0.313 & 0.067 \\
\bottomrule
\end{tabular}
\end{table}

\begin{table}[htbp]
\centering
\caption{Cold-start RL performance of Qwen2.5-7B under GRPO and CFPO across increasing iteration counts. While both objectives improve reasoning behavior at moderate training lengths, GRPO collapses under aggressive iteration settings. CFPO sustains stable optimization over longer horizons, though extreme iteration counts eventually degrade performance even under quadratic regularization.}
\label{tab:qwen7b-base}
\begin{tabular}{lcccc}
\toprule
\textbf{Model} & \textbf{Math 500} & \textbf{GSM8K} & \textbf{GPQA-Diamond} & \textbf{AIME 24} \\
\midrule
Qwen2.5-7B & 0.553 & 0.636 & 0.308 & 0.100 \\
\midrule
\multicolumn{5}{l}{\textit{GRPO Training}} \\
\quad GRPO-Iteration2  & 0.730 & 0.699 & 0.303 & 0.100 \\
\quad GRPO-Iteration4  & 0.722 & 0.762 & 0.323 & 0.056 \\
\quad GRPO-Iteration8  & 0.710 & 0.645 & 0.323 & 0.067 \\
\quad GRPO-Iteration16 & 0.002 & 0.004 & 0.298 & 0.000 \\
\midrule
\multicolumn{5}{l}{\textit{CFPO Training}} \\
\quad CFPO-Iteration2  & 0.704 & 0.596 & 0.313 & 0.067 \\
\quad CFPO-Iteration4  & 0.711 & 0.661 & 0.257 & 0.100 \\
\quad CFPO-Iteration8  & 0.722 & 0.602 & 0.323 & 0.100 \\
\quad CFPO-Iteration16 & 0.206 & 0.059 & 0.237 & 0.000 \\
\bottomrule
\end{tabular}
\end{table}

\newpage

\section{\texttt{verl} Figures}
\begin{figure*}[h]
    \centering
    \includegraphics[width=\textwidth]{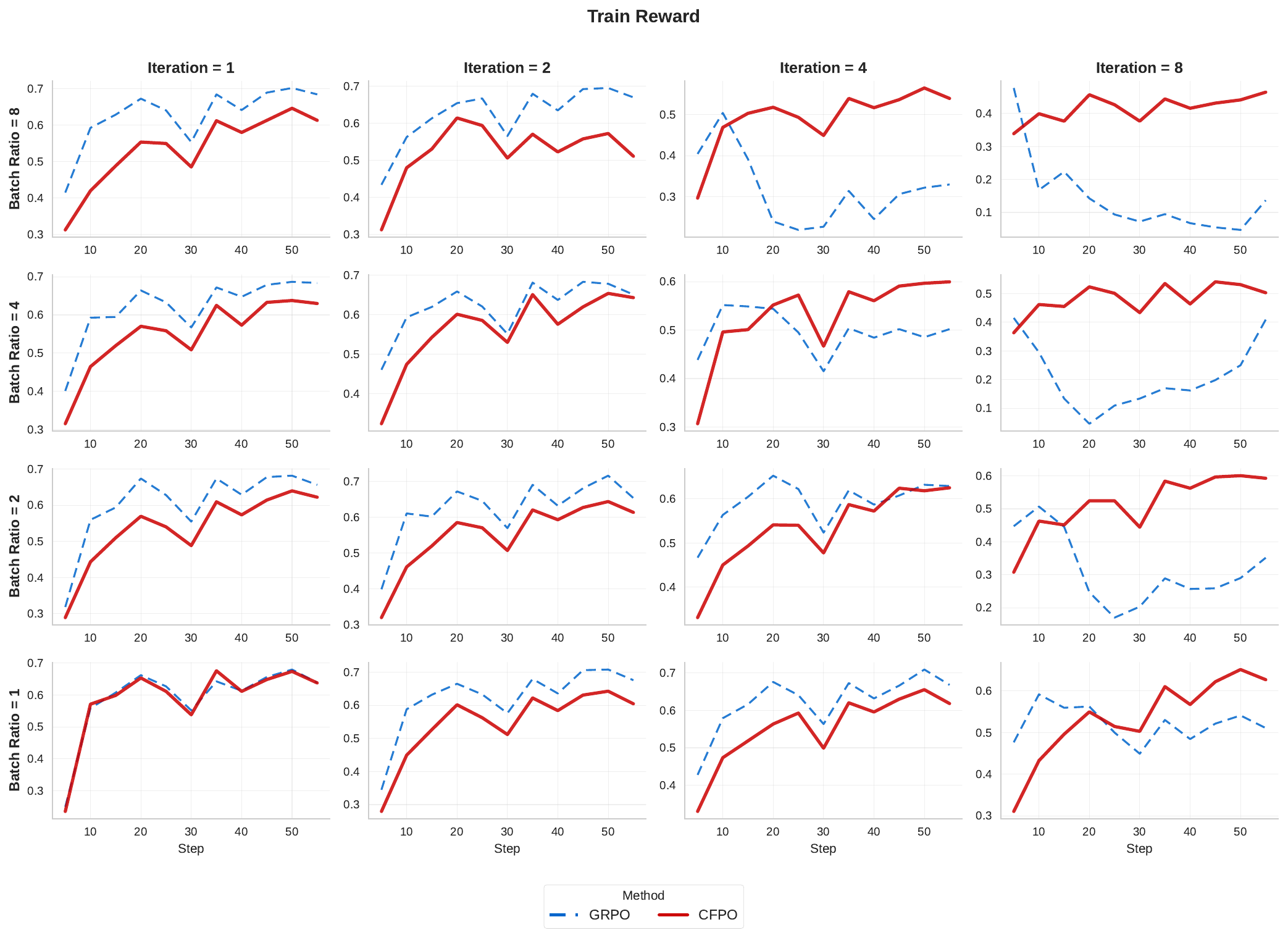}
    \caption{Training reward dynamics for cold-start RLVR training of Qwen2.5-3B using \texttt{verl} across batch ratios and iteration counts. GRPO exhibits faster early reward improvement, while CFPO progresses more gradually and converges later. Increasing iteration count leads to instability around 8 iterations, whereas increasing batch ratio alone remains comparatively stable, highlighting the stronger destabilizing effect of iteration-based sample reuse.}
    \label{fig:verl-trainreward}
\end{figure*}

\begin{figure*}[h]
    \centering
    \includegraphics[width=\textwidth]{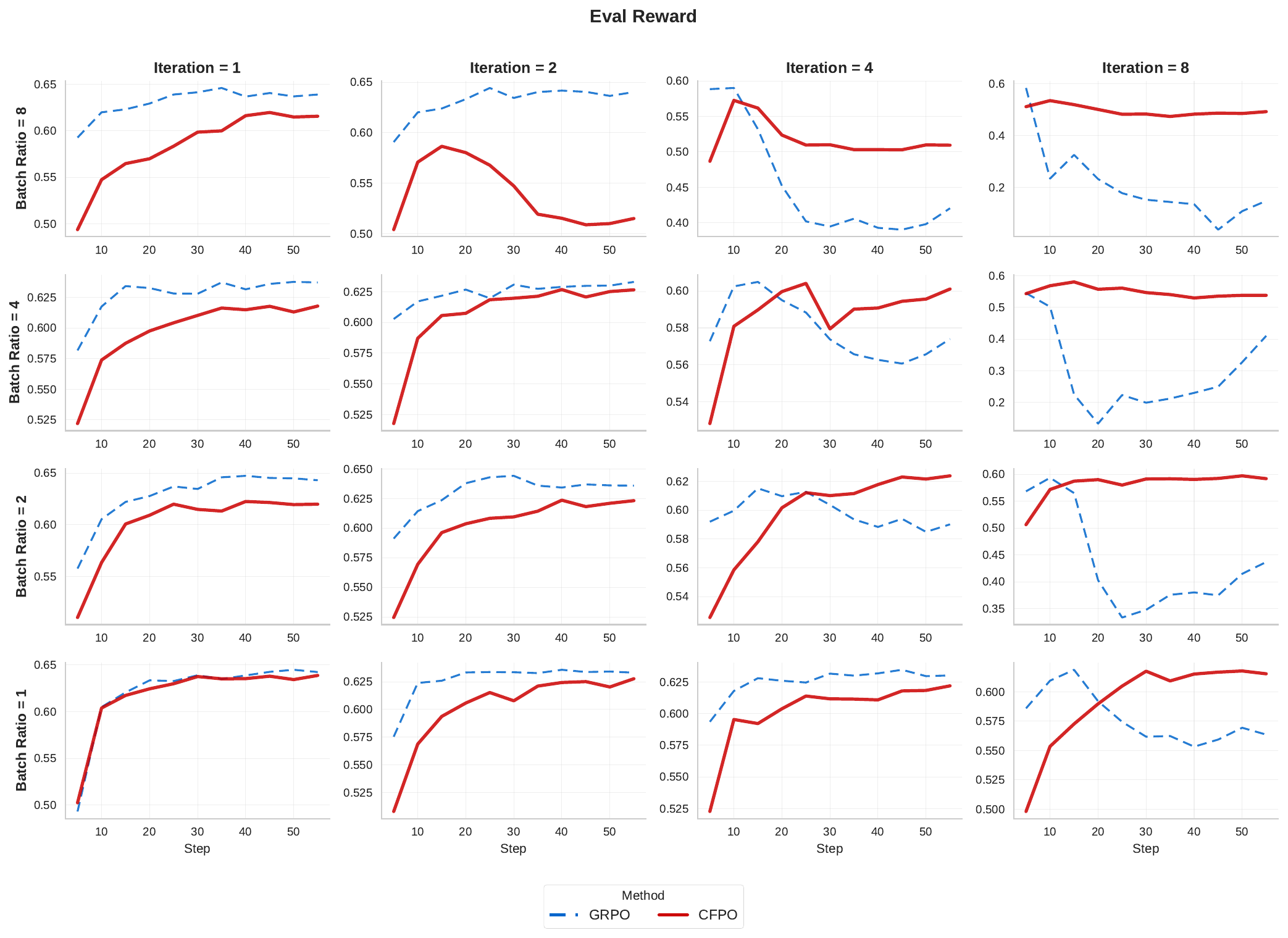}
    \caption{Validation reward during cold-start RLVR training of Qwen2.5-3B under GRPO and CFPO across batch ratios and iteration counts. Trends largely mirror training reward, with no systematic gains from increased off-policy updates. Instability emerges primarily with higher iteration counts rather than larger batch ratios, indicating limited generalization benefits from aggressive off-policy training.}
    \label{fig:verl-valreward}
\end{figure*}

\begin{figure*}[h]
    \centering
    \includegraphics[width=\textwidth]{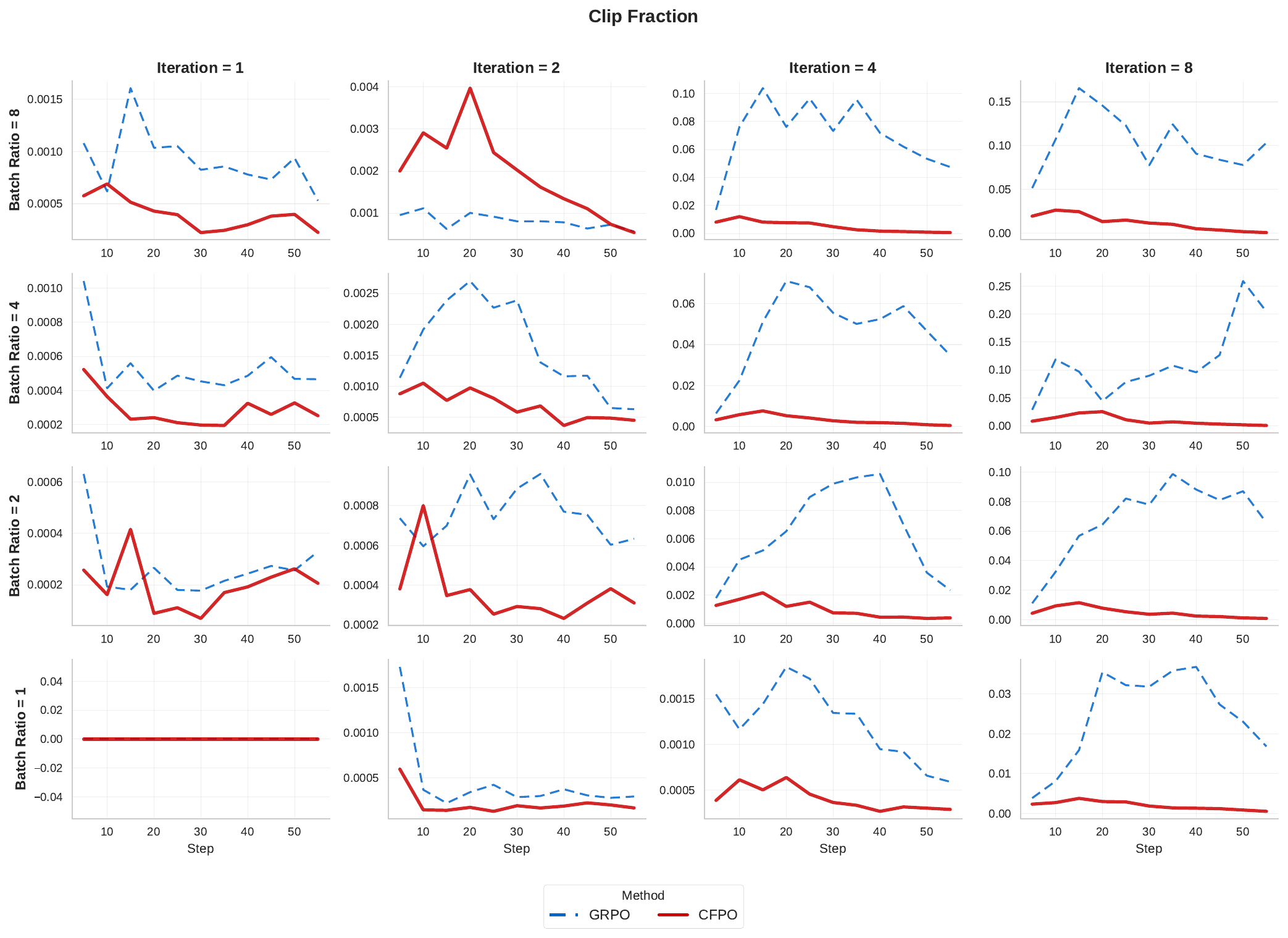}
    \caption{Policy gradient clipping ratios during cold-start RLVR training of Qwen2.5-3B using \texttt{verl}. For GRPO, clipping activity increases with both batch ratio and iteration count, and sharp rises in clipping precede training instability. CFPO consistently maintains lower and more stable clipping behavior across all settings, reflecting its smoother update geometry.}
    \label{fig:verl-pg_clipfrac}
\end{figure*}
\begin{figure*}[h]
    \centering
    \includegraphics[width=\textwidth]{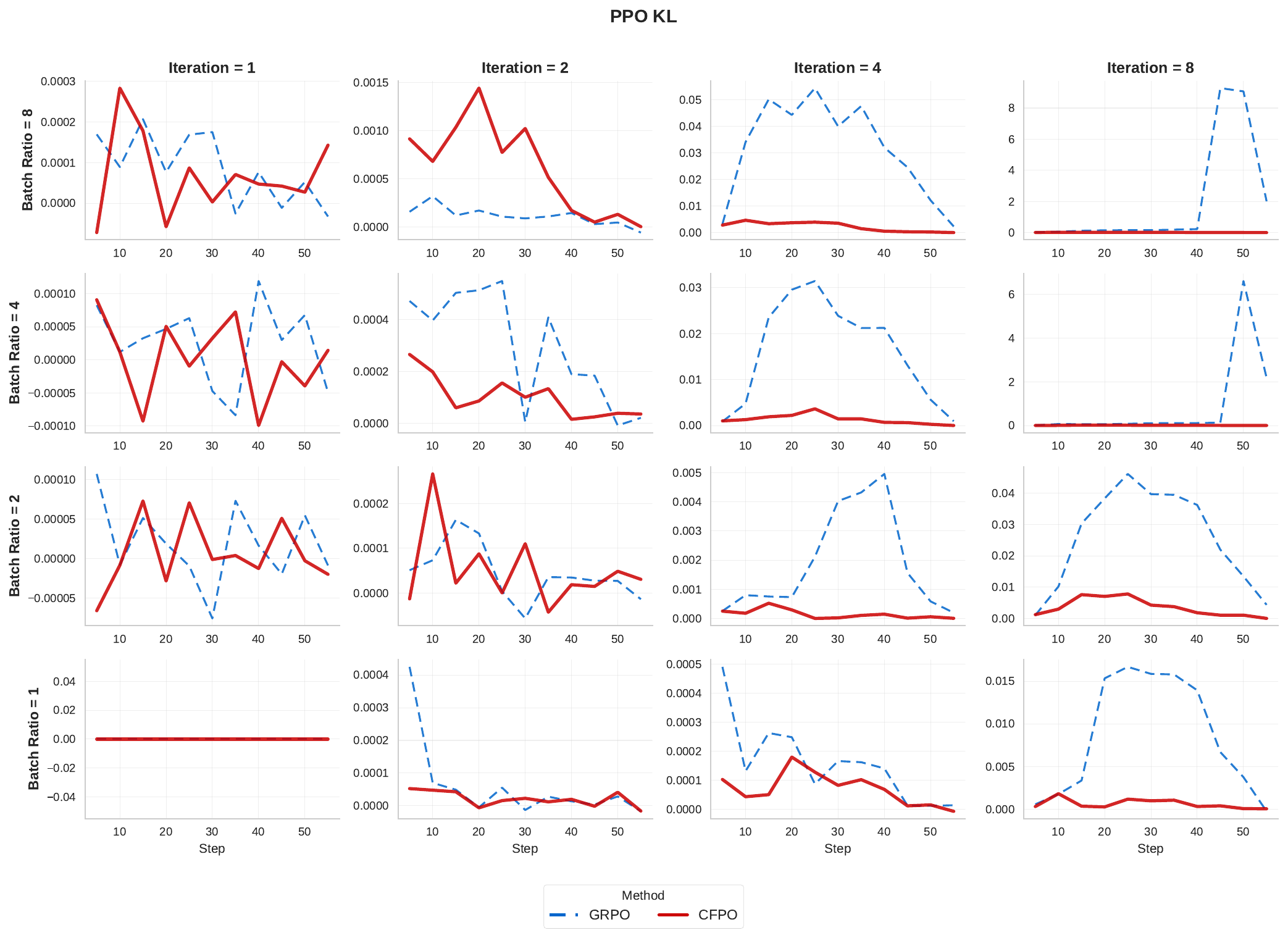}
    \caption{KL divergence between consecutive policies during cold-start RLVR training of Qwen2.5-3B across batch ratios and iteration counts. Despite differing optimization behavior, GRPO and CFPO exhibit similar KL magnitudes throughout training. This indicates that observed stability differences are not explained by large inter-policy shifts, but rather by differences in how updates are regularized within the trust region.}

    \label{fig:verl-ppo-kl}
\end{figure*}
\begin{figure*}[h]
    \centering
    \includegraphics[width=\textwidth]{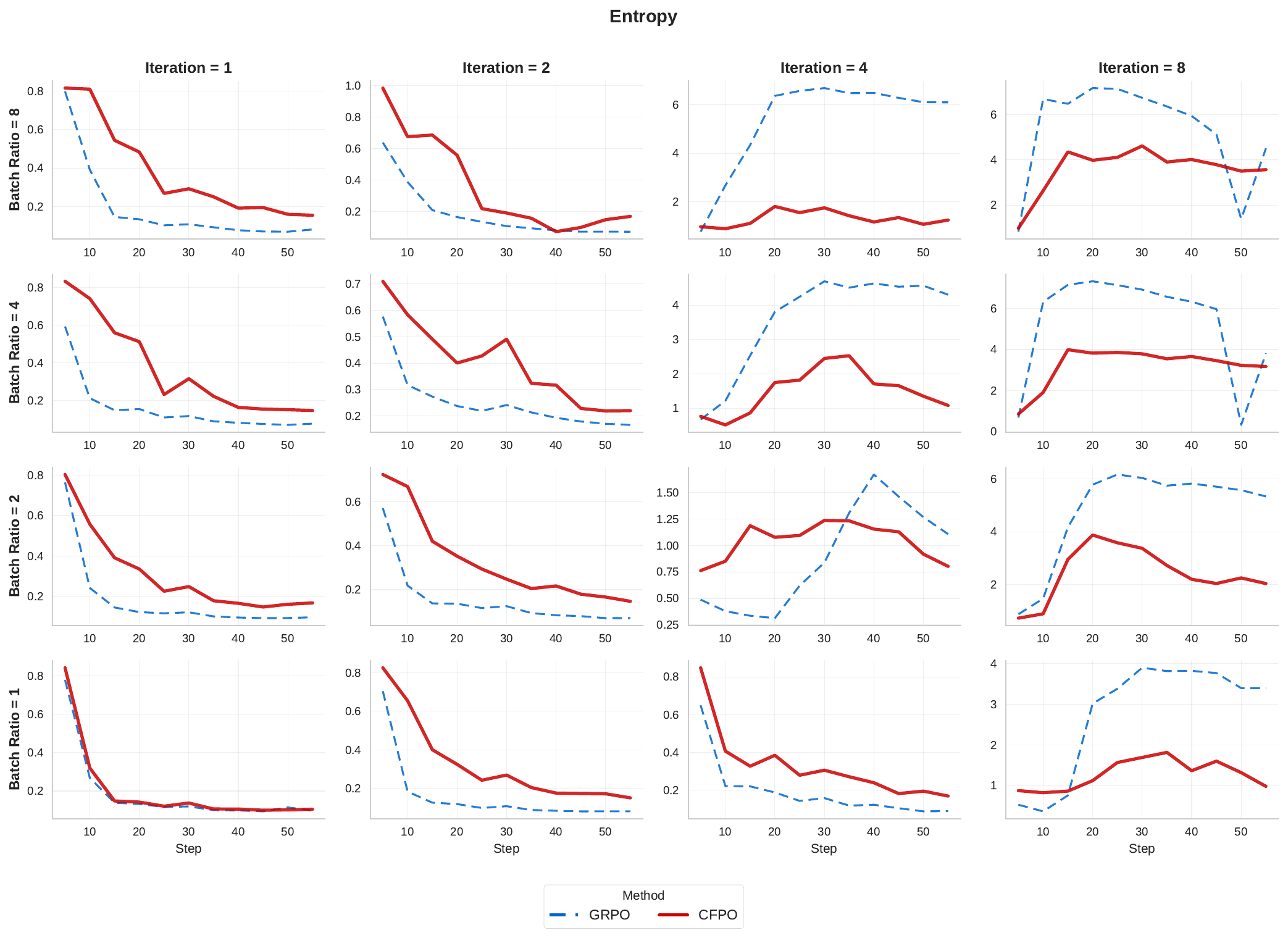}
    \caption{Policy entropy during cold-start RLVR training of Qwen2.5-3B under GRPO and CFPO across batch ratios and iteration counts. GRPO exhibits a faster reduction in entropy, particularly at higher iteration counts, consistent with its more aggressive optimization behavior. CFPO maintains higher entropy values over training, reflecting less aggressive policy updates and more gradual concentration of the policy distribution.}

    \label{fig:verl-entropy}
\end{figure*}


\end{document}